\documentclass{article}
\usepackage{ijcai25}
\usepackage{natbib}  
\usepackage{times}
\usepackage{soul}
\usepackage{url}
\usepackage[hidelinks]{hyperref}
\usepackage[utf8]{inputenc}
\usepackage[small]{caption}
\usepackage{graphicx}
\usepackage{amsmath}
\usepackage{amsthm}
\usepackage{booktabs}
\usepackage{algorithm}
\usepackage{algorithmic}
\usepackage[switch]{lineno}

\usepackage{paralist,amsmath, amssymb}
\usepackage{subfigure}
\usepackage{epstopdf}
\usepackage{booktabs} % for professional tables
\usepackage{hyperref}
\usepackage{multirow}    
\usepackage{array} 

\def \R {\mathbb{R}}

\def \K {\mathcal{K}}
\def \F {\mathcal{F}}

\def \G {\mathcal{G}}

\def \x {\mathbf{x}}

\def \z {\mathbf{z}}
\def \ze {\mathbf{0}}

\def \w {\mathbf{w}}

\def \y {\mathbf{y}}

\DeclareMathOperator*{\argmin}{argmin}

\newtheorem{thm}{Theorem}
\newtheorem{lem}{Lemma}

\newtheorem{assum}{Assumption}

\makeatletter
\newcommand\figcaption{\def\@captype{figure}\caption}
\newcommand\tabcaption{\def\@captype{table}\caption}
\makeatother
\hypersetup{
	pdfinfo={
		TemplateVersion={IJCAI.2025.0}
	}
}

\begin{document}
	
\title{Revisiting Multi-Agent Asynchronous Online  Optimization with Delays: \\the Strongly Convex Case}
 
 \author{
  Lingchan Bao\textsuperscript{\rm 1}, Tong Wei\textsuperscript{\rm 2}, Yuanyu Wan\textsuperscript{\rm 3}\\
  \affiliations \textsuperscript{\rm 1}Southeast University-Monash University Joint Graduate School, Southeast University, Suzhou 215123, China\\
  \textsuperscript{\rm 2}School of Computer Science and Engineering, Southeast University, Nanjing 211189, China\\
\textsuperscript{\rm 3}School of Software Technology, Zhejiang University, Ningbo 315100, China\\
\{baolc, weit\}@seu.edu.cn, wanyy@zju.edu.cn
 }

 \maketitle
 
 \begin{abstract}
 We revisit multi-agent asynchronous online optimization with delays, where only one of the agents becomes active for making the decision at each round, and the corresponding feedback is received by all the agents after unknown delays. Although previous studies have established an $O(\sqrt{dT})$ regret bound for this problem, they assume that the maximum delay $d$ is knowable or the arrival order of feedback satisfies a special property, which may not hold in practice. In this paper, we surprisingly find that when the loss functions are strongly convex, these assumptions can be eliminated, and the existing regret bound can be significantly improved to $O(d\log T)$ meanwhile. Specifically, to exploit the strong convexity of functions, we first propose a delayed variant of the classical follow-the-leader algorithm, namely FTDL, which is very simple but requires the full information of functions as feedback. Moreover, to handle the more general case with only the gradient feedback, we develop an approximate variant of FTDL by combining it with surrogate loss functions. Experimental results show that the approximate FTDL outperforms the existing algorithm in the strongly convex case. 
 \end{abstract}

	\section{Introduction}
	Online learning \citep{Online:suvery,Hazan2016} has received great attention over the past decades, due to the emergence of large-scale applications such as online advertising \citep{Tan-2012}, recommendation systems \citep{Lin2020}, and portfolio management \citep{Agarwal2006}. In general, it is formulated as a repeated game between an agent and an adversary, where the agent first selects a decision $\x_t$ from a set $\K\subseteq \R^n$ at each round $t$ and then the adversary selects a loss function $f_t(\cdot):\K\mapsto\mathbb{R}$. The performance of the agent is measured by regret 
	\[
	R_T = \sum_{t=1}^T f_t(\x_t) - \min_{\x\in\K}\sum_{t=1}^T f_t(\x)
	\]
	which is the gap between the cumulative loss of the agent and that of a fixed optimal decision over the total $T$ rounds. The most popular paradigm for online learning is online convex optimization (OCO) \citep{Zinkevich2003,Hazan2007,Abernethy08,Duchi2011,Adam-ICLR,wang-ICLR20,Wan-AAAI-2021-C,Wan-PAMI-2022}, which focuses on the case with convex loss functions and decision sets. By exploiting the general convexity, these previous studies on OCO have achieved an optimal regret bound of $O(\sqrt{T})$. Moreover, if the loss functions are strongly convex, an optimal regret bound of $O(\log T)$ can be further achieved.
	%For example, both online gradient descent \citep{Zinkevich2003} and follow the regularized leader \citep{Shai07} enjoys the optimal  regret bound
	
	However, the standard OCO assumes that the feedback about the loss function $f_t(\cdot)$ can be received immediately after the decision $\x_t$ is made, which may not hold in many practical applications, especially those with distributed scenarios. For example, in distributed sensor networks \citep{4434303,5740320,Kia2014,2938639}, multiple sensors capable of collecting, processing, and exchanging data are widely used in online detection, classification, and tracking of some targets. Given the limited capabilities of sensors, fully decentralized communication and computation are crucial to avoid overloading small sensors or overwhelming busy networks. For the same reason, only the sensor located near the target is activated to perform detection or classification at a specific moment, and communicates with its neighboring sensors for collaboration. Note that both the processes of target observation and data communication between these sensors will cause delays.  Moreover, due to the possible heterogeneity of these sensors, different sensors collaborate asynchronously and suffer different delays in general.

	Motivated by these applications, \citet{Hsieh2022} extend the standard OCO into multi-agent asynchronous online optimization with delays. The setting involves a set of agents $\mathcal{M} = \{1, \ldots, M\}$, and only one agent becomes active in each round \( t \) to make the decision $\x_t$. Due to the effect of delays and asynchronicities, each agent $i\in \mathcal{M}$ can only receive the feedback about $f_t(\cdot)$ at the end of round $t+d_{t,i}-1$, where $d_{t,i}\in\mathbb{Z}^{+}$ denotes an arbitrary delay. It is well-known that if only one agent exists, i.e., $M=1$, various algorithms have been proposed to achieve an optimal $O(\sqrt{dT})$ regret bound for convex functions and an optimal $O(d \log T)$ regret bound for strongly convex functions in the worst case \citep{Weinberger02_TIT,Joulani13,Quanrud2015,Joulani16,Wan2022}, where $d=\max_{t,i}\{d_{t,i}\}$ denotes the maximum delay. However, compared with the single-agent case, multi-agent OCO with delays faces two new challenges: the non-monotonicity of feedback sequence and lack of synchronization, which make the time-decaying and even delay-dependent learning rates utilized in these single-agent delayed algorithms unavailable~\citep{Hsieh2022}.

	To tackle these challenges, \citet{Hsieh2022} propose the first algorithm called delayed dual averaging (DDA) for the multi-agent case. If the maximum delay $d$ is knowable or the arrival order of feedback satisfies a special property, they prove that DDA can utilize only local information to adjust the learning rate and achieve the optimal $O(\sqrt{dT})$ regret bound for convex functions. However, these assumptions may also not hold in practice, and thus it is natural to ask whether these assumptions can be removed. Moreover, it remains unclear whether the strong convexity of functions can be further utilized to improve the regret as in the single-agent case. In this paper, we provide an affirmative answer to both questions by proposing new algorithms that can exploit the strong convexity of functions to reduce the regret bound without any prior information and assumption about the delay. Our key idea is to extend the classical follow-the-leader (FTL) algorithm \citep{hanan,FTPL05,Hazan2007}, which originally is parameter-free and enjoys an $O(\log T)$ regret bound for the standard OCO with strongly convex functions, into the multi-agent case. Specifically, we first consider the case with the delayed full-information feedback, and propose follow-the-delayed-leader (FTDL) by replacing all historical loss functions minimized in each round of FTL with all received loss functions. Our analysis reveals that FTDL can enjoy an $O(d\log T)$ regret bound for strongly convex functions in a parameter-free manner, similar to FTL. Furthermore, we study the more general case with only delayed gradient feedback, and further develop an approximate variant of our FTDL by combining it with gradient-based surrogate loss functions \citep{Hazan2007}. It is worth noting that this approximate variant not only enjoys the same $O(d\log T)$ regret bound for strongly convex functions, but can be even implemented more efficiently than FTDL. Finally, we conduct experiments on four publicly available datasets to demonstrate that our approximate FTDL can outperform the existing DDA algorithm in the case with strongly convex functions.

	\section{Related Work}
	In this section, we briefly review the related work about the standard OCO, and multi-agent OCO with delays.
	
	\subsection{The Standard OCO}
	Since the pioneering work of \cite{Zinkevich2003}, there has been a growing research interest in developing and analyzing algorithms for OCO \citep{Hazan2007,Abernethy08,Duchi2011,Adam-ICLR,wang-ICLR20,Wan-AAAI-2021-C,Wan-PAMI-2022}. Roughly speaking, these existing OCO algorithms can be divided into two categories. The first category includes online gradient descent (OGD) \citep{Zinkevich2003} and its variants, which can be viewed as an extension from gradient descent in offline optimization \citep{Boyd04} into the online setting. The other category contains follow-the-leader (FTL) \citep{hanan} and its variants, where the so-called ``leader'' is the optimal decision for all historical loss functions. Note that both categories are able to achieve an $O(\sqrt{T})$ regret bound for convex functions, and an $O(\log T)$ regret bound for strongly convex functions \citep{Zinkevich2003,Hazan2007,Hazan2016}. Moreover, these $O(\sqrt{T})$ and $O(\log T)$ regret bounds have been proved to be optimal for convex and strongly convex functions, respectively \cite{Abernethy08}.

	Here, we only introduce the detailed procedures about FTL and its variants, because they are more related to this paper. Specifically, at each round $t$, FTL simply chooses the following decision
	\begin{equation}
	\label{FTL-decision}
	\x_t = \argmin_{\x \in \K} \sum_{s=1}^{t-1} f_s(\x).
	\end{equation}
	Despite of the simplicity, it can enjoy the optimal $O(\log T)$ regret bound for strongly convex functions \citep{Hazan2007}. However, there are two unsatisfactory points that limit the application of FTL. First, one can prove that in the case with only general convex functions, the regret of FTL could be $\Omega(T)$ \citep{Hazan2016}. Second, to implement FTL, we need to solve an offline optimization problem per round, which may become a computational bottleneck in practice. 

	To address these two issues simultaneously, previous studies \citep{hanan,FTPL05,ShalevShwartz2006ConvexRG,Hazan2016} have proposed to apply FTL with a linearized approximation of the original loss functions, i.e.,
	\[
	\tilde{f}_t(\x) = f_t(\x_t) + \langle\nabla f_{t}(\x_t), \x - \x_t\rangle
	\]
	and a randomized or strongly convex regularization $\mathcal{R}(\cdot):\K\mapsto\mathbb{R}$. In this way, the decision for each round $t$ becomes
	\begin{equation*}
	\begin{split}
	% \x_t =& \argmin_{\x \in K} \sum_{s=1}^{t-1} \tilde{f}_s(\x)\\
	\x_t=&\argmin_{\x \in \K} \sum_{s=1}^{t-1} \langle\nabla f_{s}(\x_s), \x \rangle+\frac{\mathcal{R}(\x)}{\eta_t}
	\end{split}
	\end{equation*}
	where $\eta_t$ is a parameter for trading off the approximate losses and regularization. Note that when $\mathcal{R}(\cdot)$ is chosen to be $1$-strongly convex, e.g., $\mathcal{R}(\x)=\|\x\|_2^2/2$, this variant of FTL is called follow-the-regularized-leader (FTRL) or dual averaging, and can achieve the optimal $O(\sqrt{T})$ regret bound for convex functions by setting an appropriate $\eta_t$ \citep{ShalevShwartz2006ConvexRG,Hazan2016}. 

	Moreover, \cite{Hazan2007} consider $\lambda$-strongly convex functions, and propose to apply FTL with the following approximation
	\begin{equation}
	\label{surrograte-sc}
	\tilde{f}_t(\x) = f_t(\x_t) + \langle\nabla f_{t}(\x_t), \x - \x_t\rangle + \frac{\lambda}{2} \|\x - \x_t\|_2^2.
	\end{equation}
	The corresponding decision for each round $t$ can be written as follows
	\begin{equation*}
	\begin{split}
	% \x_t =& \argmin_{\x \in K} \sum_{s=1}^{t-1} \tilde{f}_s(\x)\\
	\x_t=&\argmin_{\x \in \K} \sum_{s=1}^{t-1} \left(\langle\nabla f_{s}(\x_s), \x \rangle + \frac{\lambda}{2} \|\x - \x_s\|_2^2\right).
	\end{split}
	\end{equation*}
   This variant of FTL is called follow-the-approximate-leader (FTAL) and can enjoy the optimal $O(\log T)$ regret bound for strongly convex functions \citep{Hazan2007}.

	\subsection{Multi-Agent OCO with Delays}
	The study of OCO with delays can date back to \citet{Weinberger02_TIT}, who consider the case with only one agent and a fixed delay, i.e., $d_t=d$ for any $t\in[T]$. They propose a black-box technique that maintains $d$ instances of a standard OCO algorithm and alternately utilize these instances to generate the new decision. In this way, each instance is utilized once every $d$ rounds, and actually only needs to handle a subsequence of the total $T$ rounds as in the non-delayed case. By combining this black-box technique with the optimal algorithms for OCO, one can achieve $O(\sqrt{dT})$ and $O(d\log T)$ regret bounds for convex and strongly convex functions, respectively. Note that as proved by \citet{Weinberger02_TIT}, these delay-dependent regret bounds are optimal in the worst case. Later, \citet{Joulani13} extend the black-box technique of \citet{Weinberger02_TIT} into the case with one agent and arbitrary delays. Even without prior information of delays, they achieve the same regret bounds for convex and strongly convex functions. However, these black-box techniques \citep{Weinberger02_TIT,Joulani13} require much higher memory costs than the standard OCO algorithm. To address this limitation, there has been growing research interest in developing memory-efficient delayed variants of various standard OCO algorithms \citep{Langford09,McMahan14,Quanrud2015,Joulani16,Li_AISTATS19,ICML20_Mertikopoulos,ICML21_delay,Wan-NIPS22,Wan2022,Wan-NIPS24,Wan-ICML24,Wu-WWW24}.

	Despite the great flourish of research on delayed OCO, the general case with multiple agents has rarely been investigated. As emphasized by \citet{Hsieh2022}, multi-agent OCO with delays is more challenging due to the following two reasons.
	\begin{itemize}
		\item First, these agents may suffer different communication latency, and thus can only receive the delayed feedback in an asynchronous way. Due to the asynchronicity and change of the active agent, the amount of available feedback is generally no longer monotone.
		\item Second, to promote high privacy and information security, it is common to assume that agents do not have access to a global counter that indicates how many decisions have been made at any given stage, and any non-local information.
	\end{itemize}
	To tackle these challenges, \citet{Hsieh2022} propose delayed dual averaging (DDA) that computes the decision for each round $t$ as follows
	\[
	\x_t=\argmin_{\x\in\K}\sum_{s\in\F_t}\langle\nabla f_s(\x_s),\x\rangle+\frac{\mathcal{R}(\x)}{\eta_t}
	\]
	where $\F_t$ is a set including the timestamps  of all feedback received by the $t$-th active agent. Let $|\F_t|$ denotes the cardinality of $\F_t$ (i.e., the number of feedback received up to round $t$). They first prove that DDA can achieve the $O(\sqrt{dT})$ regret bound for convex functions, by setting
	\[
	\eta_t=O\left(1/\sqrt{d\left(|\F_t|+d\right)}\right). 
	\]
	Unfortunately, the value of $d$ is generally unknown in practice, which could limit the application of DDA.\footnote{Even with the prior information of $d$, \citet{Hsieh2022} actually also need to assume that $|\F_s|\leq|\F_t|$ for any $s\in \F_t$.}  Note that in the single agent case, this issue can be solved by utilizing the standard ``doubling trick'' \citep{LEA97} to estimate the value of $d$ on the fly and then adaptively adjust the parameter $\eta$. However, it is highly non-trivial to extend this idea into the multi-agent case due to the lack of a global counter. As a result, \citet{Hsieh2022} propose an alternative mechanism inspired by data-dependent OCO algorithms, e.g., AdaGrad \citep{Duchi2011}, to adjust $\eta_t$ without prior information of $d$, but introduce an additional assumption on the arrival order of feedback. It remains unclear whether the regret of multi-agent OCO with delays can be bounded without any prior information and assumption about the delay.

	\section{Our Results}
	In this section, we first introduce some common assumptions, and then introduce our two algorithms as well as the corresponding theoretical guarantees.
	\subsection{Assumptions}
	Following previous studies \citep{Online:suvery,Hazan2016,Hsieh2022}, we first introduce two assumptions about the gradient norm of functions and the decision set.
		\begin{assum}
		\label{assum2}
		The gradients of all loss functions are bounded by some constant $G$, i.e.,
		\[\|\nabla f_t(\x)\|_2\leq G\]
		for any $\x\in\K$ and $t\in[T]$.
	\end{assum}
		\begin{assum}
		\label{assum1}
		The decision set $\K$ contains the origin $\ze$, and its diameter is bounded by some constant $D$, i.e.,
		\[\|\x-\y\|_2\leq D\]
		for any $\x,\y\in\K$.
	\end{assum}
	Moreover, as the critical difference between this paper and \citet{Hsieh2022}, we introduce the following assumption about the strong convexity of functions.
	\begin{assum}
		\label{assum3}
		All loss functions are $\lambda$-strongly convex over the decision set $\K$, i.e.,
		\[f_t(\y)\geq f_t(\x)+\langle\nabla f_t(\x),\y-\x\rangle+\frac{\lambda}{2}\|\x-\y\|^2_2\]
		for any $\x,\y\in\K$ and $t\in[T]$.
	\end{assum}

	\subsection{Follow-the-Delayed-Leader}
	Before introducing our algorithms, we first provide a formal definition for the available information held by each agent $i\in[M]$. To be precise, let $\G_i$ denote the set including the timestamps for all feedback received by each agent $i\in[M]$. In the beginning, the set should be initialized as $\G_i=\emptyset$. Then, due to the effect of delays and asynchronicity, each agent $i$ will receive the feedback from some previous round $s$, and update the set as $\G_i=\G_i\cup\{s\}$. Moreover, it will relay the received feedback if necessary. In this way, when an agent $i$ is activated at round $t$, it can only make the decision based on the feedback from rounds in the latest set $\G_i$ that may only be a subset of $[t-1]$.
	
	Therefore, even if the functions are strongly convex, FTL in Eq.~\eqref{FTL-decision} cannot be directly utilized. To address this issue, we propose a delayed variant of FTL for multi-agent OCO with delays that is called follow-the-delayed-leader (FTDL) and utilizes the following decision
	\[
	\x_t=\argmin_{\x\in\K}\sum_{s\in\F_t}f_s(\x)
	\]
	where $\F_t$ is set to be the latest $\G_i$ of the active agent $i$ at round $t$. Note that when $\F_t=\emptyset$, i.e., the active agent at round $t$ has not received any feedback yet, the decision $\x_t$ can be arbitrary in $\K$. The detailed procedures of our FTDL are summarized in Algorithm \ref{alg1} from the point of view of any single agent $i$. Despite the simplicity, we prove FTDL enjoys the following regret bound, which demonstrates that it can exploit the strong convexity automatically.
	\begin{algorithm}[tb]
		\caption{Follow-the-Delayed-Leader (Agent $i$)}
		\label{alg1}
		\begin{algorithmic}[1]
			%\STATE \textbf{Initialize:}
			\STATE \textbf{Initialize:} $\G_i=\emptyset$ and $t=1$
			\WHILE{not stopped}
			\STATE Asynchronously receive feedback $f_s(\x)$ of round $s$
			\STATE Update $\G_i=\G_i\cup\{s\}$
			\IF{the agent becomes active}
			\STATE Set $\F_t=\G_i$
			\STATE Play $\x_t=\argmin_{\x\in\K}\sum_{s\in\F_t}f_s(\x)$
			\ENDIF
			\ENDWHILE
		\end{algorithmic}
	\end{algorithm}
	\begin{thm}
		\label{thm1}
		Under Assumptions \ref{assum2} and \ref{assum3}, Algorithm \ref{alg1} ensures
		\[
		R_T\leq\frac{2dG^2(1+\ln T)}{\lambda}.
		\]
	\end{thm}
	\begin{proof}
	Inspired by the existing analysis for FTL \citep{Garber-SIAM16}, we first define an ideal decision for each round $t=1,\dots,T$ as below
	\[
     \tilde{\x}_t=\argmin_{\x\in\K}\sum_{s=1}^{t}f_s(\x)
	\]
	which actually is even one step ahead of FTL in Eq.~\eqref{FTL-decision}. In the literature, it is well-known that such an ideal decision enjoys non-positive regret for any loss functions $f_1(\x),\dots,f_T(\x)$ over $\K$.
	\begin{lem}
	\label{lm1}
			(Lemma 6.6 in \citet{Garber-SIAM16}) Let $f_1(\x),\dots,f_T(\x)$ be a sequence of loss functions over $\K$, and let $\tilde{\x}_t=\argmin_{\x\in\K} \sum_{i=1}^{t}f_i(\x)$, then
			\[
			\sum_{t=1}^T f_t(\x_{t})-\min_{\x\in\K}\sum_{t=1}^T f_t(\x) \leq0.
			\]
	\end{lem}
	Moreover, the regret of our FTDL can be divided into two parts: one is the regret of the ideal decision, and the other is the gap between the cumulative loss of the ideal decision and our FTDL, as follows
		\begin{equation}
			\label{eq0}
			\begin{aligned}
				R_T=&\sum_{t=1}^T f_t(\x_t) - \sum_{t=1}^T f_t(\x^\ast)\\\
				=&\sum_{t=1}^T \left(f_t(\x_t) -f_t(\tilde{\x}_{t})\right)+ \sum_{t=1}^T \left(f_t(\tilde{\x}_{t})-f_t(\x^\ast)\right)
			\end{aligned}
		\end{equation}
	where $\x^\ast=\argmin_{\x\in\K}\sum_{t=1}^Tf_t(\x)$.

	By combining Lemma \eqref{lm1} with Eq.~\eqref{eq0}, it is easy to verify that
	\begin{equation}
			\label{eq1}
			\begin{aligned}
				R_T\leq& \sum_{t=1}^T \left(f_t({\x}_{t})-f_t(\tilde{\x}_{t})\right)\leq\sum_{t=1}^T\langle\nabla f_t({\x}_t),{\x}_{t}-\tilde{\x}_{t}\rangle\\
				\leq &\sum_{t=1}^T\|\nabla f_t({\x}_t)\|_2\|{\x}_{t}-\tilde{\x}_{t}\|_2\leq G\sum_{t=1}^T\|\x_t-\tilde{\x}_{t}\|_2
			\end{aligned}
		\end{equation}
	where the last inequality is due to Assumption \ref{assum2}.

	Then, we only need to bound the distance between our decision and the ideal decision. To this end, we introduce a nice property of strongly convex functions from \citet{Hazan2012}.
	\begin{lem}
		\label{lambda-strongly}
		Assume $f(\x)$ is $\lambda$-strongly convex over $\K$, it holds that
		\[
		\|\x - \x^*\|^2_2 \leq \frac{2}{\lambda} \left(f(\x) - f(\x^*)\right)
		\]
		for $\x^*=\argmin_ {\x\in\K} f(\x)$ and any $\x\in\K $.
		% Since a minimizer satisfies the property that $0\in\nabla f_t(\x^*)$, we can easily prove the result by combining with Assumption \ref{assum3}. 
	\end{lem}
	Let $F_t (\x) =\sum_{s=1}^{t}f_s(\x)$ for any $t=1,\dots,T$. Note that $F_t (\x)$ is $(t\lambda)$-strongly convex over $\K$ due to Assumption \ref{assum3}. As a result, by combining Lemma \ref{lambda-strongly} with the definition of $\tilde{\x}_t$ and $F_t(\x)$, we have
		\begin{align*}
			&\|\x_t-\tilde\x_{t}\|^2_2\\\leq& \frac{2}{t\lambda} \left( F_t(\x_t)-F_t(\tilde{\x}_{t})\right)\\
			=& \frac{2}{t\lambda}\left(\sum_{s\in\F_t} f_s(\x_t)-\sum_{s\in\F_t} f_s(\tilde{\x}_{t})\right)\\&+ \frac{2}{t\lambda}\left(\sum_{s\in [t]\setminus\F_t} f_s(\x_t) -\sum_{s\in [t]\setminus\F_t} f_s(\tilde{\x}_{t}) \right)\\\leq& \frac{2}{t\lambda}\left(\sum_{s\in [t]\setminus\F_t} f_s(\x_t) -\sum_{s\in [t]\setminus\F_t} f_s(\tilde{\x}_{t})\right)\\\leq&\frac{2}{t\lambda}\sum_{s\in [t]\setminus\F_t} \|\nabla f_s(\x_t)\|_2\|\x_t-\tilde\x_{t}\|_2\\\leq&\frac{2}{t\lambda}G\left(t-|\F_t|\right)\|\x_t-\tilde\x_{t}\|_2
		\end{align*}
		where the second inequality is due to the definition of $\x_t$.

		By combining the above inequality with the fact that any feedback from rounds $1,\dots,t-d$ must have been received by any agent $i$ at the end of round $t-1$, we have
		\begin{equation}
			\label{eq2}
			\begin{aligned}
				\|\x_t-\tilde\x_{t}\|_2\leq\frac{2G(t-|\F_t|)}{t\lambda}\leq\frac{2dG}{t\lambda}.
			\end{aligned}
		\end{equation}
		Finally, by substituting (\ref{eq2}) into (\ref{eq1}), we have
		\[
		R_T\leq\frac{2dG^2}{\lambda}\sum_{t=1}^T\frac{1}{t}\leq\frac{2dG^2(1+\ln T)}{\lambda}.
		\]
	\end{proof}
	\paragraph{Remark 1.}~First, from Theorem~\ref{thm1}, our FTDL can exploit the strong convexity of functions to achieve an $O(d\log T)$ regret bound, which is much tighter than the $O(\sqrt{dT})$ regret bound achieved by the DDA algorithm \citep{Hsieh2022}. More interestingly, our FTDL does not require any prior information and assumption about the delay. Second, Assumption~\ref{assum1} is mainly introduced for making Assumption~\ref{assum2} rigorous. Actually, it is worth noting that the regret bound in Theorem \ref{thm1} does not explicitly depend on the radius of $\K$ given by Assumption \ref{assum1}. The nice property actually inherits from the standard OCO with strongly convex functions \citep{Hazan2007}. Besides, we notice that this property naturally enables an extension of our algorithm to unconstrained domains, which will be detailed in Section 3.4.
	
	\subsection{Approximate Follow-the-Delayed-Leader}
	Although our FTDL has already achieved the desired regret bound, two unsatisfactory points still exist, which could possibly limit its application. First, FTDL requires full-information feedback, i.e., the total information about each loss function $f_t(\x)$. However, in general, one may only be able to receive the gradient of each loss function $f_t(\x)$, i.e., $\nabla f_t(\x_t)$. Second, similar to FTL, our FTDL also needs to solve an offline optimization problem to compute its decision, which may be time-consuming in practice. To address these issues, we further develop an approximate variant of FTDL, by combining it with gradient-based surrogate loss functions.

	Specifically, inspired by the approximate variant of FTL \citep{Hazan2007}, it is natural to apply our FTDL to the surrogate loss functions defined in Eq.~\eqref{surrograte-sc}, which adopts the following decision
	\begin{equation}
	\label{straightforward-AFTDL}
	\begin{split}
	% \x_t =& \argmin_{\x \in K} \sum_{s=1}^{t-1} \tilde{f}_s(\x)\\
	\x_t=&\argmin_{\x \in \K} \sum_{s\in\F_t} \left(\langle\nabla f_{s}(\x_s), \x \rangle + \frac{\lambda}{2} \|\x - \x_s\|_2^2\right).
	\end{split}
	\end{equation}
	At first glance, to compute $\x_t$ in Eq.~\eqref{straightforward-AFTDL}, these agents need to share both the gradient $\nabla f_{t}(\x_t)$ and the decision $\x_t$ with others, which doubles the communication costs of the DDA algorithm \citep{Hsieh2022}. However, it is worth noting that $\x_t$ in Eq.~\eqref{straightforward-AFTDL} is equivalent to
	\begin{equation}
	\label{s2-AFTDL}
		\x_t = \argmin_{\x \in \K} \sum_{s \in \F_t} \left( \langle \nabla f_{s}(\x_s)-\lambda\x_s, \x \rangle + \frac{\lambda}{2} \|\x\|_2^2 \right)
	\end{equation}
	because the term $(\lambda/2)\|\x_s\|_2^2$ does not affect the minimizer of the optimization problem in Eq.~\eqref{straightforward-AFTDL}. As a result, these agents only need to share one vector $\z_t=\nabla f_t(\x_t)-\lambda\x_t$  with others, which enjoys the same communication costs as the DDA algorithm \citep{Hsieh2022}. 
	Therefore, the decision $\x_t$ can be computed as
	\begin{equation*}
		\x_t = \argmin_{\x \in \K} \sum_{s \in \F_t} \left( \langle \z_s, \x \rangle + \frac{\lambda}{2} \|\x\|_2^2 \right).
	\end{equation*}

	The detailed procedures of our approximate FTDL are summarized in Algorithm \ref{alg2} from the point of view of any single agent $i$. 
	
	Now, we  demonstrate that it can also exploit the strong convexity of functions as the original FTDL.
	\begin{algorithm}[tb]
		\caption{ Approximate FTDL (Agent $i$)}
		\label{alg2}
		\begin{algorithmic}[1]
			%\STATE \textbf{Initialize:}
			\STATE \textbf{Initialize:} $\G_i=\emptyset$ and $t=1$
			\WHILE{not stopped}
			\STATE Asynchronously receive feedback $\z_s=\nabla f_s(\x_s)-\lambda\x_s$ from the round $s$
			\STATE $\G_i=\G_i\cup\{s\}$
			\STATE Relay $\z_s$ if necessary
			\IF{the agent becomes active}
			\STATE Set $\F_t=\G_i$ and play
			\[\x_t=\argmin_{\x\in\K}\sum_{s\in\F_t}\left(\langle\z_s,\x\rangle+\frac{\lambda}{2}\|\x\|^2_2\right)\]
			\ENDIF
			\ENDWHILE
		\end{algorithmic}
	\end{algorithm}
	\begin{thm}
		\label{thm2}
		Under Assumptions \ref{assum2}, \ref{assum1}, and \ref{assum3}, Algorithm \ref{alg2} ensures
		\begin{align*}
			R_T \leq \frac{2d(G+2\lambda R)^2(\ln T+1)}{\lambda}.
		\end{align*}
	\end{thm}
	\begin{proof}
		For any $t=1,\dots,T$, let \(
		\tilde{f}_t(\x)=\langle\z_t,\x\rangle+\frac{\lambda}{2}\|\x\|_2^2
		\) and \(
		F_t(\x)=\sum_{s=1}^t \tilde{f}_s(\x)\). Similar to the proof of Theorem \ref{thm1}, for each round $t$, we can define an ideal decision as below
		\begin{equation}
			\label{eq3}
			\begin{split}
				\tilde{\x}_t=&\argmin_{\x\in\K}\sum_{s=1}^{t}\tilde{f}_s(\x)=\argmin_{\x\in\K}F_t(\x).
			\end{split}
		\end{equation} 	
		According to Assumption \ref{assum3}, we have
		\begin{equation}
			\label{eq4}
			\begin{aligned}
				R_T = &\sum_{t=1}^T f_t(\x_t) - \sum_{t=1}^T f_t(\x^\ast) \\ \leq&  \sum_{t=1}^T \left(\langle \nabla f_t(\x_t), \x_t  - \x^\ast \rangle-\frac{\lambda}{2} \|\x_t - \x^\ast\|^2_2 \right)\\
				= & \sum_{t=1}^T \left( \tilde{f}_t(\x_t) - \tilde{f}_t(\x^\ast) \right)\\
				= & \sum_{t=1}^T \left( \tilde{f}_t(\x_t) - \tilde{f}_t(\tilde{\x}_t) \right)+ \sum_{t=1}^T \left( \tilde{f}_t(\tilde{\x}_t) - \tilde{f}_t(\x^\ast) \right)
			\end{aligned}
		\end{equation}
		where $\x^\ast=\argmin_{\x\in\K}\sum_{t=1}^Tf_t(\x)$.
		
		Then, we notice that the first term in the right side of Eq.~\eqref{eq4} can be bounded by analyzing the distance of our decision and the ideal decision. Specifically, under Assumptions \ref{assum2} and \ref{assum1}, for any $t=1,\dots,T$, we have
		\begin{equation}
			\label{bound-zt}
			\begin{split}
				&\|\z_t\|_2=\|\nabla f_t(\x_t)-\lambda\x_t\|_2\\
				\leq&\|\nabla f_t(\x_t)\|_2+\|\lambda\x_t\|_2\leq G+\lambda R.
			\end{split}
		\end{equation}
		Moreover, for any $t=1,\dots,T$, we can show that $\tilde{f}(\x)$ is $(G+2\lambda R)$-Lipschitz, i.e., for all $\x \in \K$, it holds that
		\begin{equation}
			\label{approx-LIPS}
			\begin{split}
				&\|\nabla \tilde{f}_t(\x)\|_2 \leq \|\z_t\|_2 + \lambda \|\x\|_2 \\&\leq (G + \lambda R) + \lambda R = G + 2\lambda R
			\end{split}
		\end{equation}
		where the second inequality is due to Eq.~\eqref{bound-zt} and Assumption~\ref{assum1}.
		
		From Eq.~\eqref{approx-LIPS}, it is easy to verify that
		\begin{equation}
			\label{dis-approx}
			\begin{split}
				\sum_{t=1}^T\left(\tilde{f}_t(\x_t) -\tilde{f}_t(\tilde{\x}_t)\right)  
				\leq (G+2\lambda R)\sum_{t=1}^T\|\x_t-\tilde{\x}_t\|_2.
			\end{split}
		\end{equation} 
		Note that $F_t (\x)$ is $(t\lambda)$-strongly convex over $\K$. As a result, by combining Lemma \ref{lambda-strongly} with the definition of $\tilde{\x}_t$ and $F_t(\x)$, we have
		\begin{align*}
			&\|\x_t-\tilde\x_{t}\|^2_2\\\leq& \frac{2}{t\lambda} \left( F_t(\x_t)-F_t(\tilde{\x}_{t})\right)\\
			=& \frac{2}{t\lambda}\sum_{s\in\F_t} \left(\tilde{f}_s(\x_t)-\tilde{f}_s(\tilde{\x}_{t})\right)\\&+ \frac{2}{t\lambda}\sum_{s\in [t]\setminus\F_t}\left( \tilde{f}_s(\x_t) -\tilde{f}_s(\tilde{\x}_{t}) \right)\\\leq& \frac{2}{t\lambda}\sum_{s\in [t]\setminus\F_t}\left( \tilde{f}_s(\x_t) -\tilde{f}_s(\tilde{\x}_{t}) \right)\\\leq&\frac{2\left(t-|\F_t|\right)(G+2\lambda R)}{t\lambda}\|\x_t-\tilde\x_{t}\|_2
		\end{align*}
		where the second inequality is due to the definition of $\x_t$, and the last inequality is due to Eq.~\eqref{approx-LIPS}.
		
		By combining the above inequality with the fact any feedback from rounds $1,\dots,t-d$ must have been received by any agent $i$ at the end of round $t-1$, we have
		\begin{equation}
			\label{eq2-appro}
			\begin{aligned}
				\|\x_t-\tilde\x_{t}\|_2\leq\frac{2d(G+2\lambda R)}{t\lambda}.
			\end{aligned}
		\end{equation}
		Now, we only need to bound the last term in the right side of Eq.~\eqref{eq4}. Specifically, by applying Lemma \ref{lm1} with the definition of $\tilde{\x}_t$, it is not hard to verify that
		\begin{equation}
			\label{eq5}
			\begin{split}
				\sum_{t=1}^T\left(\tilde{f}_t(\tilde{\x}_t)-\tilde{f}_t(\x^\ast)\right)
				\leq 0
			\end{split}
		\end{equation}
		where the last inequality is also due to Assumption \ref{assum1}.
		
		Finally, by combining Eq.~\eqref{eq4} with Eq.~\eqref{dis-approx}, Eq.~\eqref{eq2-appro}, and Eq.~\eqref{eq5}, we
		have
		\begin{equation*}
			\begin{split}
				R_T\leq& \sum_{t=1}^T\frac{2d(G+2\lambda R)^2}{t\lambda}\\
				\leq&\frac{2d(G+2\lambda R)^2(\ln T+1)}{\lambda}.
			\end{split}
		\end{equation*}
	\end{proof}
	\paragraph{Remark 2.}~ From Theorem \ref{thm2}, the approximate FTDL enjoys the same $O(d\log T)$ regret bound as FTDL, and also does not require any prior information and assumption about the delay. However, we want to clarify that our A-FTDL actually also relies on the delay via the set $\F_t$. Although such a reliance is not easily noticeable, it actually plays a role similar to a delay-dependent step-size used in DDA~\citep{Hsieh2022}. Moreover, we emphasize that although Theorem~\ref{thm2} requires the boundedness of decision sets, i.e., Assumption \ref{assum1}, it is actually easy to extend our Algorithm~\ref{alg2} to further handle the unbounded set by using the fact that the optimal solution of a strongly convex function must lie in a ball centered at the origin, as detailed in Section 3.4.
	
	\paragraph{Remark 3.}~ We want to clarify that in A-FTDL the argmin subproblem that we need to solve is equal to a projection operation onto the convex set $\K$, i.e., 
	\begin{align*}
		\x_t &=\argmin_{\x\in\K}\sum_{s\in\F_t}\left(\langle\z_s,\x\rangle+\frac{\lambda}{2}\|\x\|^2_2\right)\\\
		& = \argmin_{\x\in\K}\left\{\frac{\lambda}{2}\left\|\x + \frac{\sum_{s\in\F_t} \z_s}{\lambda}\right\|^2_2 - \frac{1}{2\lambda}\left\|\sum_{s\in\F_t} \z_s\right\|^2_2 \right\}\\\
		&= \argmin_{\x\in\K}\left\|\x + \frac{\sum_{s\in\F_t} \z_s}{\lambda}\right\|^2_2\
	\end{align*}
	This computational cost is also required by the previous DDA algorithm~\citep{Hsieh2022}, and many other online optimization algorithms.
	
	\begin{figure*}[t]
		\centering
		\subfigure[ijcnn1]{	\includegraphics[width=0.48\textwidth]{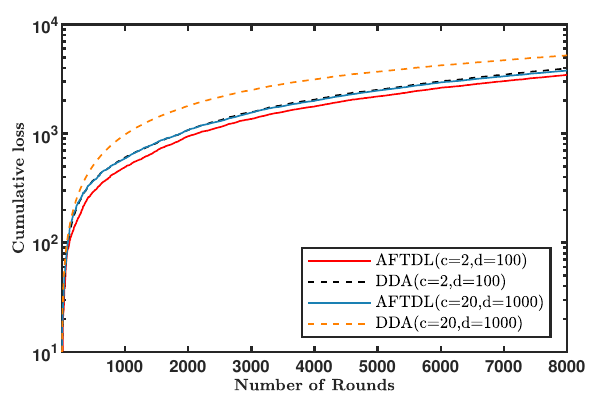}}
		\subfigure[w8a]{\includegraphics[width=0.48\textwidth]{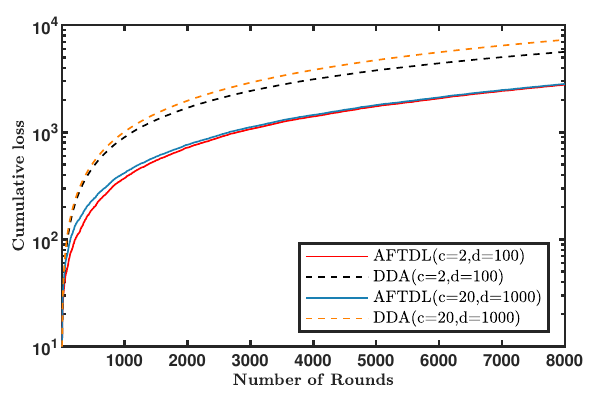}}
		\\
		\subfigure[phishing]{\includegraphics[width=0.48\textwidth]{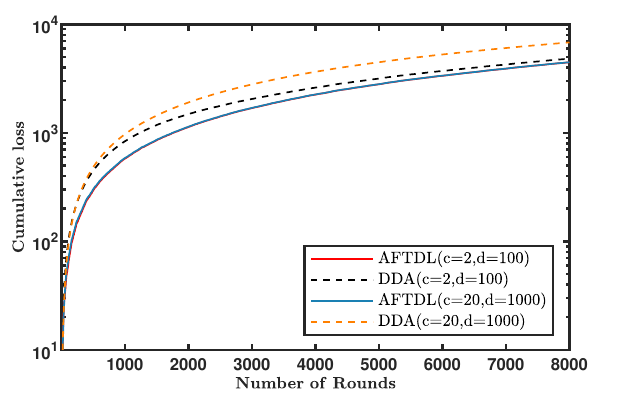}}
		\subfigure[a9a]{\includegraphics[width=0.48\textwidth]{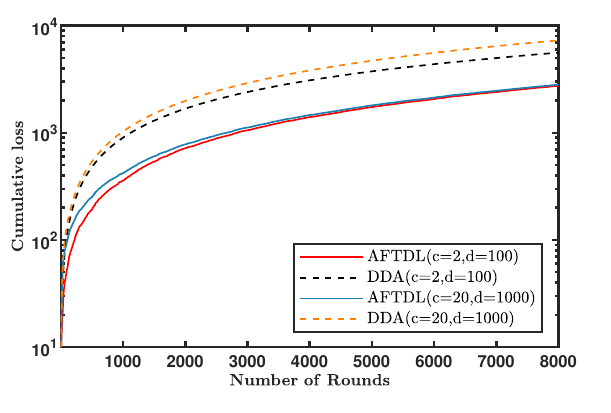}}
		\caption{Comparisons of our A-FTDL against DDA on delayed $2$-agent online binary classification.}
		\label{fig}
	\end{figure*}
	\subsection{Extension to Unconstrained Domains}
	Furthermore, we consider the unconstrained setting where $\mathcal{K} = \mathbb{R}^n$, and extend our algorithm to handle this case while maintaining strong theoretical guarantees. Note that such an extension follows from existing studies on delayed OCO with strongly convex functions~\citep{Wan-NIPS24}. Here, we provide the details for completeness. To be precise, in unconstrained domains, the original Assumptions 1 and 2 may not hold globally. To address this challenge, we introduce two weaker and more reasonable assumptions.
	\begin{itemize}
		\item All loss functions are $G$-Lipschitz continuous at the origin, i.e., $|f_t(\mathbf{0}) - f_t(\x)| \leq G\|\x\|_2$ for all $\x$ in a neighborhood of $\mathbf{0}$.
		
		\item Each loss function $f_t$ is $\lambda$-strongly convex over $\mathbb{R}^n$.
	\end{itemize}
	
	Under these two weaker assumptions, by applying Lemma \ref{lambda-strongly}, it is not hard to verify that,
	\begin{equation}
		\|\x^*\|_2 \leq \frac{2}{\lambda}\left(f_t(\mathbf{0}) - f_t(\x^*)\right) \leq \frac{2G}{\lambda}
	\end{equation}
	where $\x^* = \argmin_{\x \in \mathbb{R}^n} \sum_{t=1}^T f_t(\x)$. As a consequence, the player only needs to select actions from the constrained set:
	
	\begin{equation}
		\mathcal{K}' = \left\{ \x \in \mathbb{R}^n \middle| \|\x\|_2 \leq \frac{2G}{\lambda} \right\} 
	\end{equation}
	which naturally satisfies Assumption 2 with radius $R = 2G/\lambda$. We further assume that all loss functions satisfy the original Assumptions 1 and 3 over the restricted set $\mathcal{K}'$. This allows us to reduce the unconstrained problem to a constrained one over $\mathcal{K}'$. Theorem \ref{thm1} shows that the regret bound of FTDL is independent of the domain's radius and therefore remains valid. For A-FTDL, we derive the regret bound from Theorem~\ref{thm2} by applying Algorithm~\ref{alg2} over a bounded domain of radius $\frac{2G}{\lambda}$, then we have
	\[
	R_T \leq \frac{50dG^2(\ln T+1)}{\lambda}.
	\]

	\section{Experiments}
	In this section, we conduct simulation experiments on four publicly available datasets---ijcnn1, w8a, phishing and a9a from the LIBSVM repository \citep{Chang2011} to compare the performance of our approximate FTDL (A-FTDL) against the DDA algorithm \citep{Hsieh2022}. The details of these datasets are summarized in Table \ref{tab1}. Both algorithms are implemented with Matlab and tested on a laptop with 2.6GHz CPU and 16GB memory.
	
	\begin{table}[t]
		\centering
		\begin{tabular}{ c | c | c |c}
			\hline
			Dataset & \# Examples & \# Features (i.e., $n$) & \# Classes \\ \hline
			ijcnn1 & 49990 & 22 & 2\\
			w8a & 49749 & 300 & 2\\
			phishing & 11055 & 68 & 2\\
			a9a & 32561  & 123 &2 \\
			\hline
		\end{tabular}
		\captionsetup{justification=centering} 
		\caption{Datasets used in our experiments.}
		\label{tab1}
	\end{table}
	Specifically, we randomly select $10000$ examples from the original data set. Then we use the first $8000$ data points as the training set and the subsequent $2000$ data as the test set. Thus, the total number of rounds $T$ in our experiments is $8000$.
	
	We consider a delayed multi-agent online binary classification problem with strongly convex loss functions. Under two scenarios: (i) a system with $2$ agents and a maximum delay of $100$, and (ii) a system with $20$ agents and a larger maximum delay of $1000$. These settings are inspired by the practical observation that as the number of agents increases, the communication and coordination overhead often grows, leading to longer feedback delays. In each round $t\in[T]$, one of the agents will be active, and choose a decision $\x_t$ from the following convex set
	\[\K = \{ \x \in \R^{n} \mid \|\x\|_2 \leq 1 \}\]
	which satisfies Assumption~\ref{assum1} with $R = 1$, where $n$ is equal to the number of features in the dataset. Then, an example $(\w_t,y_t)$ is selected from the dataset, and causes a loss
	\[
	f_t(\x_t)=\max \left\{ 1-y_t\w_t^\top\x_t, 0 \right\} + \frac {\lambda}{2}\|\x_t\|^2_2 
	\]
	where $\w_t\in \R^{n}$ denotes the feature vector, $y_t\in \{ -1, 1\}$ denotes the class label, and $\lambda=0.01$. It is easy to verify that the loss function $f_t(\x)$ satisfies Assumption \ref{assum3}, and ensures
	\[\|\nabla f_t(\x)\|_2 \leq \lambda\|\x\|_2 + \|\w_t\|_2.\] 
	Note that by a single pass of the datasets, we can compute the maximum norm of $\w_t$, and denote this value as $w_{\max}$. Then, the above loss function $f_t(\x)$ also satisfies Assumption \ref{assum2} with
	\[G = \lambda R + w_{\max}=0.01+w_{\max}.
	\]
	Next, to simulate the delays, we independently generate random delay sequences for each agent. For the 2-agent scenario, each agent's delays are randomly sampled from the range $[1,100]$. For the 20-agent scenario, each agent's delays are randomly sampled from the range $[1,1000]$. Moreover, we recall that the DDA algorithm needs to choose an appropriate $\mathcal{R}(\x)$ and tune the parameter $\eta_t$. According to Proposition 7 of~\cite{Hsieh2022}, we set $\mathcal{R}(\x)=(1/2)\|\x\|_2^2$ and use
	\[\eta_t = \frac{R}{\sqrt{2}G\sqrt{(1+2d)(|\F_t|+d+1)}}\]
	which ensures the $O(\sqrt{dT})$ regret bound. 

	Fig.~\ref{fig} shows the comparisons of our A-FTDL and the DDA algorithm~\citep{Hsieh2022} on the two settings. We observe that in both settings, the cumulative loss of A-FTDL is consistently lower than that of DDA, verifying the theoretical improvement on regret by exploiting strong convexity. Moreover, as the number of agents increases and the delays become larger, the performance gap between A-FTDL and DDA widens, further demonstrating the benefit of strong convexity for removing the need for prior information and assumptions about the delay.

	Table~\ref{tab:runtime} shows that the running times of A-FTDL and DDA are comparable under both settings. Particularly, in the 2-agent setting with smaller delays, the difference between the two algorithms is even smaller, which is consistent with our previous complexity analysis.
	
	Finally, we evaluate the test accuracy of the final decision vector $\x_T$ obtained by each algorithm on the $2000$ test examples. The evaluation is conducted under the setting of 2-agent with $d_{\max} = 100$ and 20-agent with $d_{\max} =  1000$. From Table~\ref{tab:accuracy}, our A-FTDL has higher testing accuracy than DDA in two datasets: a9a and w8a. In the other two datasets, A-FTDL and DDA have the same testing accuracy. We want to emphasize that these results are reasonable, because the testing accuracy actually is not directly determined by the regret bound.
	\begin{table}[t]
	\centering
	\begin{tabular}{c|c|c|c|c}
		\hline
		\multicolumn{1}{c|}{Dataset}  & \multicolumn{2}{c|}{2 agents, $d=100$} & \multicolumn{2}{c}{20 agents, $d=1000$} \\
		\cmidrule(lr){2-3} \cmidrule(lr){4-5}
		& A-FTDL & DDA & A-FTDL & DDA \\
		\hline
		a9a      & 0.1970 & 0.1465 & 1.6503 & 1.2387 \\
		ijcnn1   & 0.1338 & 0.1112 & 1.2645 & 1.1594 \\
		phishing & 0.1271 & 0.1377 & 1.3161 & 1.2222 \\
		w8a      & 0.2271 & 0.1494 & 1.5526 & 1.1928 \\
		\hline
	\end{tabular}
	\captionsetup{justification=centering} 
	\caption{Average per-round running time (seconds) comparison on different datasets}
	\label{tab:runtime}
\end{table}

\begin{table}[t]
	\centering
	\begin{tabular}{c|c|c|c|c}
		\hline
		\multicolumn{1}{c|}{Dataset}  & \multicolumn{2}{c|}{2 agents, $d=100$} & \multicolumn{2}{c}{20 agents, $d=1000$} \\
		\cmidrule(lr){2-3} \cmidrule(lr){4-5}
		& A-FTDL & DDA & A-FTDL & DDA \\
		\hline
		a9a       & 83.70\% & 76.05\% &82.55\%&76.00\%\\
		ijcnn1    & 90.05\% & 90.05\% &89.80\%&89.80\%\\
		w8a       & 89.80\% & 89.65\% &89.80\%&89.60\%\\
		phishing  & 56.60\% & 56.60\% &56.40\% &56.40\%\\
		\hline
	\end{tabular}
	\captionsetup{justification=centering} 
	\caption{Test accuracy (\%) comparison on  on different datasets}
	\label{tab:accuracy}
\end{table}	
 
\section{Conclusion and Future Work}
In this paper, we investigate multi-agent asynchronous online optimization with delays, and develop two novel algorithms called FTDL and A-FTDL for the case of strongly convex functions. According to the analysis, our two algorithms exploit the strong convexity to achieve an $O(d\log T)$ regret bound, which is much tighter than the existing $O(\sqrt{dT})$ regret bound achieved by only using the convexity. More interestingly, different from the existing algorithm that needs to know the maximum delay beforehand or requires an additional assumption on the arrival order of feedback, our two algorithms do not require any prior information or assumption about the delay. Finally, numerical experiments show that our A-FTDL outperforms the existing algorithm in real applications.

An open question is whether our algorithms and analysis can be further extended to a broader class of exponentially concave (exp-concave) functions. We notice that in the standard OCO, there is also a variant of FTL~\citep{Hazan2007} that enjoys an $O(n\log T)$ regret bound for exp-concave functions. Very recently,~\cite{qiu25a} have proposed algorithms for delayed OCO with exp-concave functions, achieving a regret bound of $O(d n \log T)$. However, the extension to multi-agent asynchronous online optimization seems still highly non-trivial. Therefore, we leave this extension for future work.
\section*{Acknowledgements}
This work was partially supported by National Natural Science Foundation of China (62306275) and Yongjiang Talent Introduction Programme (2023A-193-G). The authors would also like to thank the anonymous reviewers for their helpful comments.

\bibliographystyle{named}	
\bibliography{ref}

\end{document}